\documentclass[runningheads]{llncs}

\usepackage{eccv}

\usepackage{eccvabbrv}

\usepackage{graphicx}
\usepackage{booktabs}

\usepackage[accsupp]{axessibility}

\usepackage{hyperref}

\usepackage{orcidlink}

\newcommand{\OurName}{\textit{DCARL}}

\begin{document}

\title{DCARL: A Divide-and-Conquer Framework for Autoregressive Long-Trajectory Video Generation}

\titlerunning{DCARL}

\author{Junyi Ouyang\inst{1,2} \and
Wenbin Teng\inst{1,2} \and
Gonglin Chen\inst{1,2} \and
Yajie Zhao\inst{1,2} \and
Haiwei Chen\inst{1,2}}

\authorrunning{J.~Ouyang et al.}

\institute{Institute for Creative Technologies \and
University of Southern California \\
\email{\{junyiouy, wenbinte, gonglinc\}@usc.edu, zhao@ict.usc.edu, chw9308@hotmail.com}}

\maketitle

\begin{abstract}
Long-trajectory video generation is a crucial yet challenging task for world modeling, primarily due to the limited scalability of existing video diffusion models (VDMs). Autoregressive models, while offering infinite rollout, suffer from visual drift and poor controllability. To address these issues, we propose \OurName{}, a novel divide-and-conquer, autoregressive framework that effectively combines the structural stability of the divide-and-conquer scheme with the high-fidelity generation of VDMs. Our approach first employs a dedicated Keyframe Generator trained without temporal compression to establish long-range, globally consistent structural anchors. Subsequently, an Interpolation Generator synthesizes the dense frames in an autoregressive manner with overlapping segments, utilizing the keyframes for global context and a single clean preceding frame for local coherence. Trained on a large-scale internet long trajectory video dataset, our method achieves superior performance in both visual quality (lower FID and FVD) and camera adherence (lower ATE and ARE) compared to state-of-the-art autoregressive and divide-and-conquer baselines, demonstrating stable and high-fidelity generation for long trajectory videos up to 32 seconds in length. Project page: \url{https://junyiouy.github.io/projects/dcarl/}
\end{abstract}    
\section{Introduction}
\label{sec:intro}

\begin{figure}[t]
    \centering
    \captionsetup{type=figure}
    \includegraphics[width=0.95\textwidth]{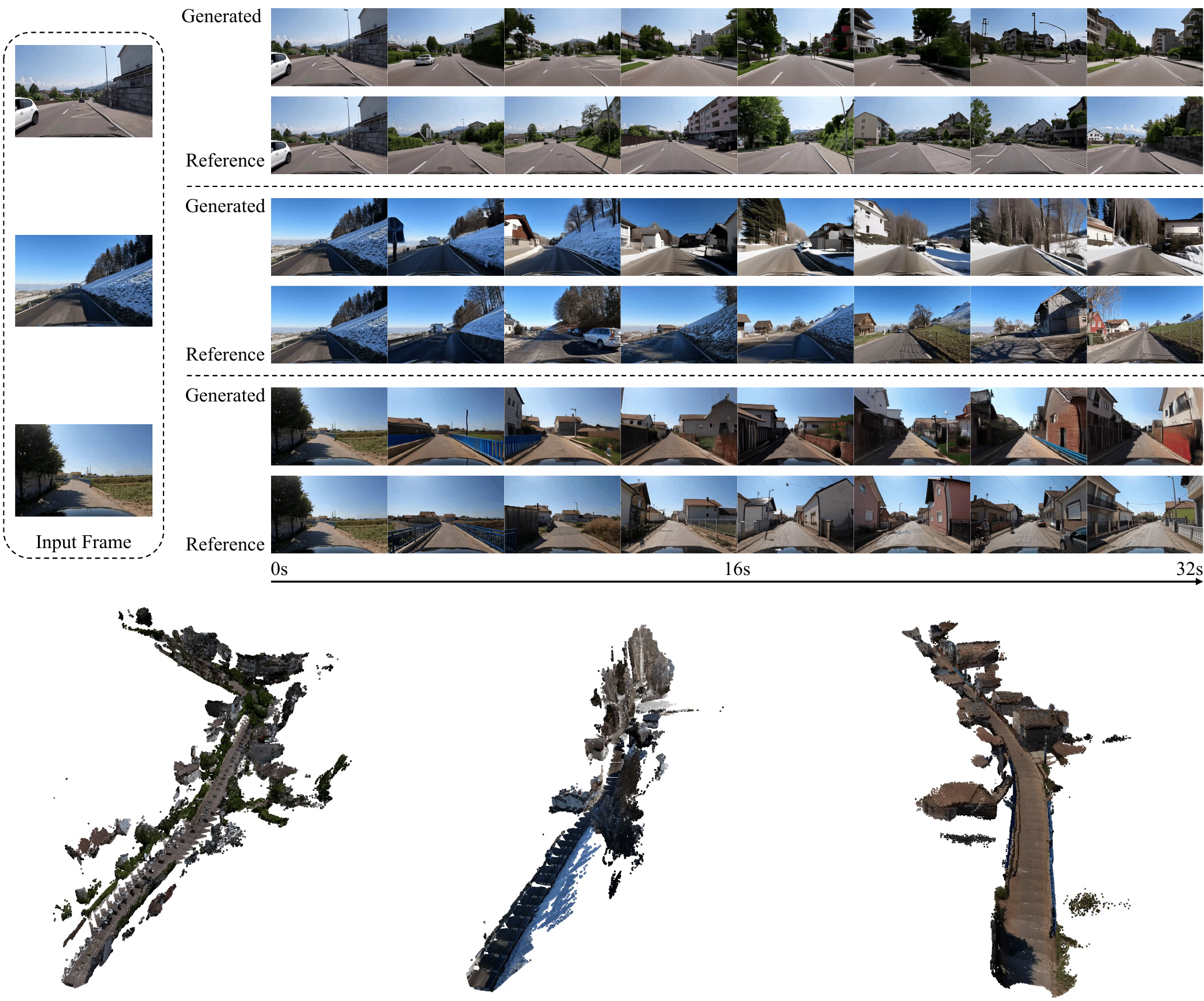}
    \captionof{figure}{
Given a single \textit{Input Frame} (left), a camera trajectory and a text description(omitted for brevity), \OurName{} generates \textbf{high-fidelity} and \textbf{temporally consistent} future scene sequences spanning $32\text{s}$.
The figure showcases synthesis results across three diverse environments: urban (top), snowy suburban (middle), and rural (bottom). The bottom panel shows the corresponding 3D point cloud reconstructions by $\pi_3$\cite{pi3} of the generated sequences. The \textit{Generated} sequences exhibit strong visual realism and long-term temporal coherence that closely match the quality of \textit{Reference} sequences, highlighting our method's efficacy for long-trajectory generation. }
    \label{fig:teaser}
\end{figure}

Long trajectory video generation has emerged as a pivotal research direction in recent years~\cite{gao2023magicdrive, epona, mila, mdv2, vista, streetcrafter, russell2025gaia}. Developing generative models that can simulate large-scale, realistic scenes is of paramount importance for terrain reconstruction, autonomous driving systems and world models. To achieve this, existing methods~\cite{cameractrl,he2025cameractrl2,recammaster,yu2025context} typically leverage video diffusion models (VDMs) to facilitate high-fidelity generation. However, while these models perform well on room-sized scenes, they scale poorly to the synthesis of extended sequences (e.g., 30-second videos). This is primarily because predicting the entire sequence collectively incurs a prohibitive computational cost, limiting their effectiveness for long-duration and large-scene tasks.

Autoregressive (AR) diffusion models~\cite{df1,pyramidal_flow,selfforcing} introduce causal attention mechanisms to VDMs, enabling long-sequence rollout through sliding-window inference and sophisticated noise scheduling. However, these models often suffer from severe drift due to exposure bias—a phenomenon arising from the discrepancy between training on ground-truth contexts and inferring from self-generated frames~\cite{framepack, mila}. These compounding prediction errors not only degrade visual quality but also severely compromise the effectiveness of external control signals like camera poses. As the synthesized content deviates from the valid manifold, the model increasingly loses its ability to interpret and adhere to the geometric constraints provided by the control trajectory. This decoupling creates a fundamental misalignment between the intended camera instructions and the actual synthesized motion, ultimately causing purely AR methods to fail in maintaining trajectory consistency over long trajectories.

Existing methods seek to alleviate exposure bias by mitigating the training--inference mismatch. Diffusion Forcing~\cite{df1,df2} injects random Gaussian noise into conditional inputs, while Self Forcing~\cite{selfforcing} trains causal models on self-generated rollouts to narrow the distribution gap. Although these strategies mitigate local drifting, they fail to explicitly bound error propagation, leaving the challenge of stable frame quality and precise camera control on long trajectories largely unaddressed. Furthermore, techniques~\cite{deep_forcing, rollingforcing} that incorporate attention sinks~\cite{streamingllm} can further impede controllability; fixed sink frames often introduce sink collapse~\cite{lol} and severely hinder the model's responsiveness to dynamic pose instructions, ultimately degrading camera motion fidelity and scene transitions across extended trajectories.

Therefore, in this work, we propose \OurName{}, a novel framework that utilizes VDM in an autoregressive norm. 
Our method follows a divide-and-conquer scheme, employing dual generators specifically tailored for the synthesis of global keyframes and local dense frames. Our key insight is that purely autoregressive generation inherently suffers from unbounded error accumulation. To overcome this, our framework first employs a full-sequence generator to jointly synthesize sparse keyframes across entire video timeline. By modeling these global keyframes jointly rather than sequentially, we effectively bypass autoregressive error propagation, establishing structurally accurate and semantic-consistent anchors for the entire long sequence.

Building upon these global anchors, our dense frame generator synthesizes the intermediate video segments via a constrained interpolation process. Specifically, the local generation is conditioned on both the surrounding keyframes, which serve as fixed constraints, and the preceding dense segment. We identify two failure modes in keyframe-interpolation-based video generation. First, the model can become overly anchored to the provided keyframes, resulting in keyframe duplication in their temporal neighborhood rather than synthesizing smooth, physically plausible motion, which yields unnatural dynamics. The second is boundary inconsistency, where the video exhibits abrupt transitions at segment boundaries. To address these issues, we propose \emph{Motion-Inductive Noisy Conditioning} and a \emph{Seamless Boundary Consistency} module, enabling more dynamic and natural long-video generation. As validated by our empirical results, our  modeling of global keyframe anchors and local transitions achieves highly stable generation of high-fidelity, long trajectory videos with precise camera control. Notably, our approach successfully mitigates the severe visual drifts lie in existing autoregressive approaches, enabling controllable long trajectory video generation.
\section{Related Works}
\label{sec:formatting}

\noindent\textbf{Autoregressive Video Generation.}
Autoregressive video generation models synthesize a long sequence in a causal manner by predicting the next frame (or block) conditioned on past generated content. Therefore, they naturally support streaming output and long-horizon rollout, but they also amplify exposure bias, error accumulation, and long-context constraints. Recent works~\cite{df1,df2,selfforcing,selfforcing++,rollingforcing,causalforcing} focuses on training-time forcing or rollout alignment to reduce the train–test mismatch. They replace purely teacher-forced learning with varying degrees of self-conditioned rollouts, aiming to make the training distribution closer to inference and thereby suppress drift over long trajectories. A second line of works~\cite{deep_forcing,longlive,infinity_rope} targets inference-time long-context stabilization and efficiency without, or with minimal, retraining. Although these methods achieve high-fidelity video generation and precise motion control, they fall short in long horizon scene generation task as the rapid scene changes increase extra challenges for pure autoregressive generation.

\noindent\textbf{Long Horizon Scene Generation.}
Long-horizon scene generation targets temporally consistent, controllable rollouts that remain plausible over hundreds to thousands of frames, where the core challenges are compounding drift, long-context scaling, and multi-view consistency. One line of work extends diffusion world models to make long rollouts more stable and controllable: VISTA~\cite{vista} and GAIA~\cite{russell2025gaia} emphasize richer conditional interfaces and design choices that better preserve structure across iterative sampling, while Epona~\cite{epona} explicitly frames long-horizon synthesis as autoregressive diffusion and trains the model to tolerate its own rollout errors. A second family pursues drift-bounding objectives for interactive generation: LIVE~\cite{live} constrains long-term deviation through consistency-style training signals, aiming to prevent unbounded error accumulation without relying solely on teacher forcing.
\section{Preliminary}
\label{sec:preliminary}

\noindent\textbf{Autoregressive Generation.} 
Long-horizon video generation is typically formulated as an iterative factorization of the joint distribution conditioned on the control signals $\mathcal{C}$:
\begin{equation}
    p(\mathcal{V} | \mathcal{C}) = \prod_{i=1}^{S} p(S_i | S_{<i}, \mathcal{C}),
\end{equation}
where the sequence $\mathcal{V}$ is partitioned into $S$ contiguous blocks $\{S_1, \dots, S_S\}$ of size $B$. This unified formulation encompasses various granularities of generation depending on the choice of $B$: \textit{token-level} AR where $B$ represents sub-frame units such as spatial-temporal patches, \textit{frame-level} AR where $B=1$ for sequential frame-wise rollout, and \textit{segment-level} AR where $B>1$ for simultaneous synthesis of multiple frames. 

Architecturally, these paradigms employ intra-block bidirectional attention for temporal correlation and inter-block causal dependency for scalability. While a larger block size $B$ improves generation quality and consistency by expanding the bidirectional receptive field, it increases training overhead in memory and computation, necessitating a balanced design for scalable long-video synthesis.

\noindent\textbf{Divide-and-Conquer Paradigm.} 
An alternative strategy for long-video synthesis involves a hierarchical decomposition of the generation task into global structural planning and local detail refinement. This paradigm is typically formulated by introducing a sparse structural proxy $\mathcal{K}$, such as keyframes or semantic layouts, which abstracts the temporal evolution of the sequence:
\begin{equation}
    p(\mathcal{V} | \mathcal{C}) = p(\mathcal{K} | \mathcal{C}) p(\mathcal{V} | \mathcal{K}, \mathcal{C}).
\end{equation}
Establishing the global trajectory via $\mathcal{K}$ transforms dense synthesis into a constrained interpolation task. This hierarchy bounds prediction errors within structural anchors, preventing the unbounded drift typical of autoregressive rollout. While decoupling long-range structure from local synthesis enables scalable generation, the paradigm remains sensitive to the noise inherent in the structural proxy $\mathcal{K}$ and requires high-fidelity conditional mapping in $p(\mathcal{V} | \mathcal{K}, \mathcal{C})$ to maintain overall consistency.
\section{Methods}
\label{sec:methods}
\begin{figure}[t]
  \centering
   \includegraphics[width=0.999\linewidth]{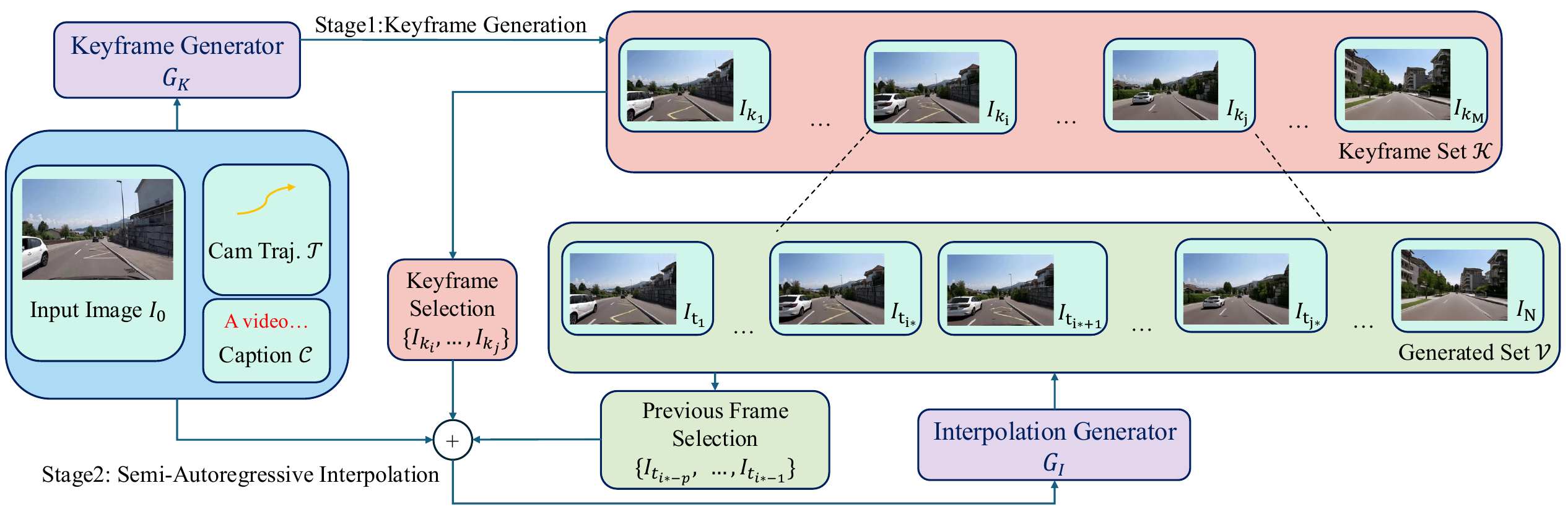}

    \caption{
    \textbf{Our Two-Stage Autoregressive Pipeline for Long-Term Video Synthesis.}
    Our model synthesizes a long video sequence $\mathcal{V}=\{I_{1}, \dots, I_N\}$ conditioned on multimodal inputs: the initial image $I_0$, camera trajectory $\mathcal{T}$, and caption $\mathcal{C}$. \textbf{The Keyframe Generator} $G_K$ produces a global Keyframe Set $\mathcal{K}$, establishing the overall structural evolution of the video. \textbf{The Interpolation Generator} $G_I$ generates the final frames in an autoregressive manner.}
   \label{fig:main_pipeline}
\end{figure}



\subsection{Overview}
\label{subsec:formulation}

We define the task of long-trajectory scene generation as generating a dense video sequence $\mathcal{V} = \{I_1, I_2, \dots, I_N\}$ of length $N$, conditioned on an initial frame $I_0$ and a sequence of control signals (e.g., camera poses) $\mathcal{T} = \{c_1, c_2, \dots, c_N\}$. The goal is to model the conditional distribution $p(\mathcal{V} | I_0, \mathcal{T})$.

\label{sec:divide}

As illustrated in Fig.~\ref{fig:main_pipeline}, our framework follows a two-stage divide-and-conquer pipeline. In the first stage, the keyframe generator $G_K$ synthesizes a sparse keyframe set $\mathcal{K}=\{I_{k_1}, \dots, I_{k_M}\}$ conditioned on the initial image $I_0$, camera trajectory $\mathcal{T}$, and caption $\mathcal{C}$. These keyframes establish the overall structural evolution of the video. In the second stage, the interpolation generator $G_I$ synthesizes the dense sequence $\mathcal{V}$ in a segment-wise manner. 

Following the autoregressive paradigm, $G_I$ synthesizes each segment $S_i$ by generating its constituent frames simultaneously. The generation is conditioned on a selected subset of keyframes $\{I_{k_i}, \dots, I_{k_j}\}$ and a history of $p$ previously generated frames $\{I_{t_i^* - p}, \dots, I_{t_i^* - 1}\}$ with $t_i^*$ being the first frame index of the current segment. This design decouples global structural planning from local detail interpolation, effectively reducing cumulative drift while maintaining precise adherence to the given trajectory.

\subsection{Keyframe Generation}
\label{sec:keyframe_gen}

The keyframe generator $G_K$ generates the keyframe set $\mathcal{K}$ as defined in ~\cref{sec:divide}. Specifically, we implement this mapping through a DiT-based\cite{dit} flow matching process $\Phi_K$:
\begin{equation}
    \mathcal{K} = \mathcal{D}\big(\Phi_K(\epsilon \mid \mathcal{E}(I_0), \mathcal{P}, \mathcal{C})\big),
\end{equation}
where $\epsilon \sim \mathcal{N}(0, \mathbf{I})$ is the initial Gaussian noise. Here, $\mathcal{E}$ and $\mathcal{D}$ denote the VAE encoder and decoder, and $\mathcal{P} = \{c_{k_1}, \dots, c_{k_M}\}$ represents the sampled poses from the global trajectory $\mathcal{T}$. Our insight is generating keyframes would achieve bounded or mitigated error accumulation as illustrated by the following proposition:

\begin{proposition}[Temporal Degradation via Perturbation Analysis]
\label{prop:perturbation_error}
Let $e_t = I_t - I_t^*$ denote the cumulative generation error at frame $t$, where $I_t$ and $I_t^*$ are the generated and ground-truth frames. Modeling the autoregressive (AR) generation as $I_{t+1} = \mathcal{F}(I_t, \epsilon_t)$, where $\mathcal{F}$ is the generator mapping and $\epsilon_t$ is the sampling noise. Assuming $\mathcal{F}$ is $L$-Lipschitz continuous with respect to the state ($L = \sup \|\nabla_I \mathcal{F}\| \ge 1$), a first-order perturbation analysis yields the cumulative error bound $\|e_N^{base}\| \leq \sum_{j=1}^N L^{N-j} \|\eta_j\|$, where $\eta_j$ represents the local injection error at step $j$. If $\|\eta_j\|$ is bounded by a constant $\eta$, this cumulative error diverges at least linearly $\mathcal{O}(N)$ (and exponentially if $L > 1$). Our divide-and-conquer strategy mitigates this via keyframe anchoring; the interpolation error at any frame $\|e_t^{ours}\|$ is bounded by the convex combination of its anchoring keyframe errors plus a local constant.
\end{proposition}


Please refer to the supplementary material for complete proof. To ensure structural integrity and control precision, we incorporate the following designs:

\noindent\textbf{Spatial-Structural Preservation.} Standard temporal VAEs typically employ temporal compression, which is inherently lossy and degrades visual quality for sparse sequences with significant spatial displacement. To circumvent this, we treat the $M$ keyframes as a batch of independent images by reshaping the frame dimension into the batch dimension during encoding and decoding. This ensures that each generated keyframe maintains full spatial fidelity, providing a reliable foundation for subsequent interpolation.

\noindent\textbf{Keyframe Sampling Strategy.} We construct the structural skeleton by sampling keyframes from the sequence using a temporal stride $\Delta k$. To ensure robustness across diverse motion scales, $\Delta k$ is randomly drawn from a discrete set of candidate strides during training. During testing, a constant stride $\Delta k_{\text{test}}$ is maintained. The selection of $\Delta k_{\text{test}}$ serves as a critical hyperparameter that modulates the trade-off between the computational efficiency of the keyframe generation process and the precision of the structural guidance provided to the interpolation model.


\subsection{Interpolation Generation}
\label{sec:interpolation_gen}


The interpolation generator $G_I$ synthesizes the dense video sequence $\mathcal{V}$ in a segment-wise manner. Specifically, for each segment $S_i$, we construct an input latent sequence $\mathbf{z}_{in}$ by concatenating the noise with temporal context. The DiT-based flow matching process $\Phi_I$ is formulated as:
\begin{equation}
\begin{aligned}
    S_i &= G_I(\mathcal{H}_i, \mathcal{K}_i, \mathcal{T}_i, \mathcal{C}) \\
    &= \mathcal{D}\big(\Phi_I(\mathbf{z}_{in}, \mathcal{T}_i, \mathcal{C})\big),
\end{aligned}
\end{equation}
where $\mathbf{z}_{in} = [\epsilon, \mathbf{z}_{\mathcal{H}_i}, \mathbf{z}_{\mathcal{K}_i}]$ represents the sequence formed by concatenating the initial Gaussian noise $\epsilon$, the encoded history $\mathbf{z}_{\mathcal{H}_i} = \mathcal{E}(\mathcal{H}_i)$, and the selected keyframes $\mathbf{z}_{\mathcal{K}_i} = \mathcal{E}(\mathcal{K}_i)$ along the temporal dimension.

\noindent\textbf{Keyframe Selection Policy.} For each segment $S_i = \{I_{t_i^*}, \dots, I_{t_j^*}\}$ and the corresponding trajectory $\mathcal{T}_i = \{c_{t_i^*}, \dots, c_{t_j^*}\}$ in the interpolation stage, we select a local subset of keyframes $\mathcal{K}_i \subset \mathcal{K}$ to provide temporal context. This subset encompasses all keyframes within the segment's temporal span, along with the nearest keyframes that immediately precede and follow the segment:
\begin{equation}
    \mathcal{K}_i = \{ I_k \in \mathcal{K} \mid k_{pre} \leq k \leq k_{next} \},
\end{equation}
where $k_{pre} = \max_{k: I_k \in \mathcal{K}, k < t_i^*} k$ and $k_{next} = \min_{k: I_k \in \mathcal{K}, k > t_j^*} k$. This policy ensures that each segment is synthesized with both look-back and look-ahead structural constraints, preventing trajectory drift at the segment boundaries.


\noindent\textbf{Autoregressive Rollout.} To synthesize long-trajectory videos, $G_I$ operates via a sliding window. During each sampling step of the denoising process, the noisy latent $\mathbf{z}$ (which originates from $\epsilon$) is continuously updated while being anchored by the fixed historical tokens $\mathbf{z}_{\mathcal{H}_i}$ and global keyframe tokens $\mathbf{z}_{\mathcal{K}_i}$ within $\mathbf{z}_{in}$. This hybrid conditioning allows the model to leverage both local continuity and global anchors in a unified attention space, which effectively mitigates exposure bias and error accumulation compared to purely autoregressive methods, while maintaining superior adherence to the camera trajectory $\mathcal{T}_i$.

\noindent\textbf{Motion-Inductive Noisy Conditioning.} \label{sec:noisy_keyframe}A common failure mode in interpolation is keyframe duplication, where the model replicates pixels from keyframes instead of synthesizing motion. This identity mapping shortcut occurs because clean keyframes $\mathbf{z}_{\mathcal{K}_i}$ allow the model to minimize training loss by anchoring to visual guides rather than learning complex temporal transitions, leading to stuttering or unnatural transitions. To mitigate this, we perturb the keyframe latents with a consistent noise level during both training and inference:
\begin{equation}
    \tilde{\mathbf{z}}_{\mathcal{K}_i} = \alpha_{c} \mathbf{z}_{\mathcal{K}_i} + \sigma_{c} \epsilon, \quad \epsilon \sim \mathcal{N}(\mathbf{0}, \mathbf{I}),
\end{equation}
where the coefficients $\alpha_{c}$ and $\sigma_{c}$ are hyperparameters that control the keyframe guidance and the noise intensity, respectively. This perturbation breaks the pixel-level shortcut, forcing $\Phi_I$ to prioritize underlying motion dynamics and trajectory adherence.


\noindent\textbf{Seamless Boundary Consistency.} To eliminate visual fractures at segment junctions, we implement a latent substitution strategy consistently during training and inference. Specifically, we define a $p$-frame overlap between segments and replace the first $p$ latents of segment $S_i$ with historical latents $\mathbf{z}_{\mathcal{H}_i}$ at each denoising step $\tau$. To ensure generalization to these deterministic boundaries, $\mathbf{z}_{\mathcal{H}_i}$ is kept noise-free during $\Phi_I$ training. By bridging the distribution gap, this design enforces strict continuity and prevents flickering or motion jumps, ensuring a seamless temporal flow across all segment boundaries.
\section{Experiments}
\begin{figure}[!t]
  \centering
  \includegraphics[width=0.9\linewidth]{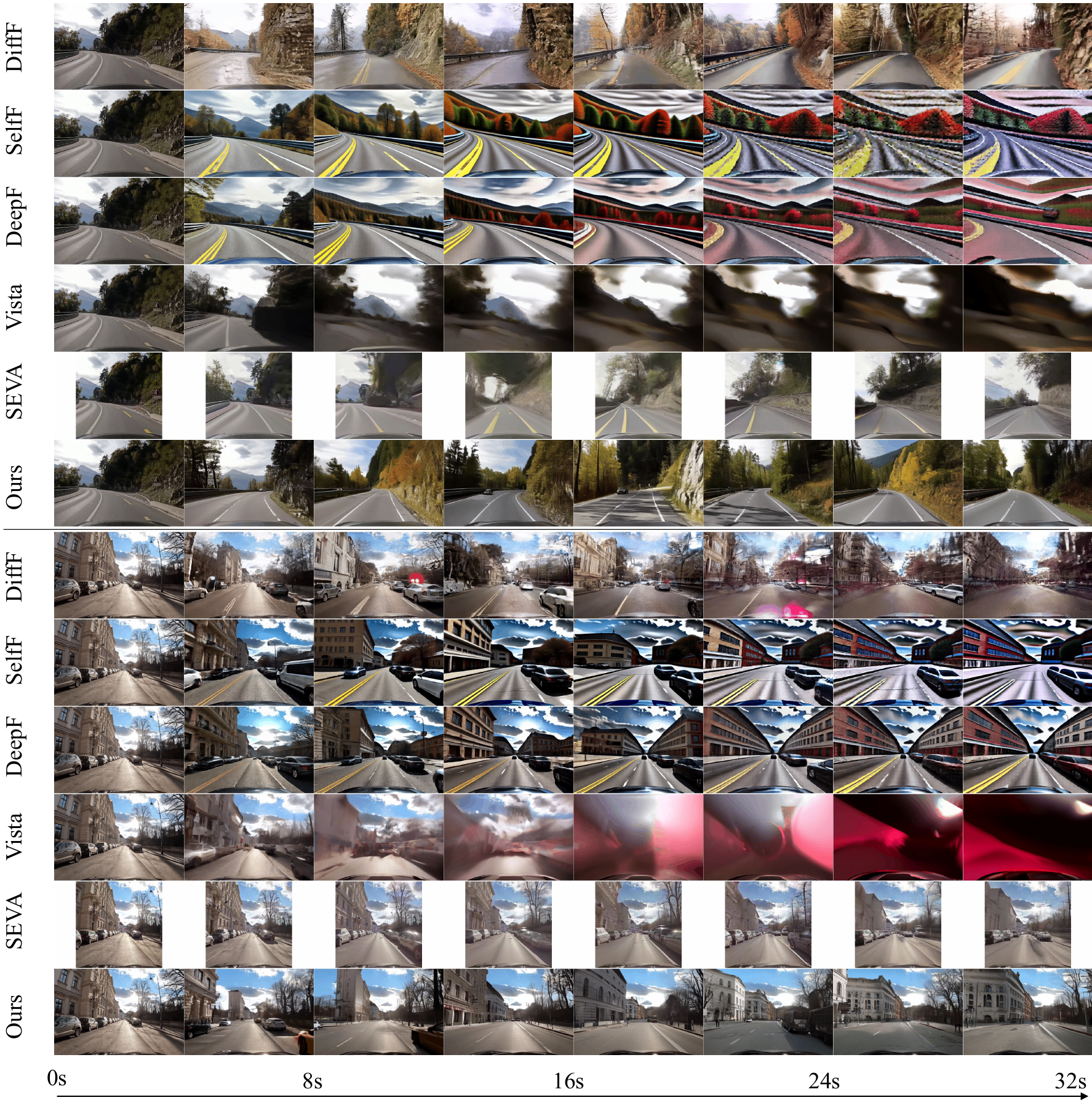}
  \caption{Qualitative comparison of long-trajectory video generation. We compare our method against several state-of-the-art baselines. Each method is provided with the same initial frame as input. Our approach demonstrates superior long-trajectory consistency and visual fidelity over a 32-second duration.} 
  \label{fig:main_qual_res} 
\end{figure}
\subsection{Experiments Setup}

\noindent\textbf{Implementation Details.} Our model is built upon the Wan2.1-T2V-1.3B \cite{wan_2_1} architecture. We integrate camera control by adding camera feature maps to visual tokens within each DiT block, following the ReCamMaster \cite{recammaster} design. Both keyframe and interpolation generators are trained for 30,000 steps using the AdamW optimizer (LR: $5 \times 10^{-5}$, weight decay: 0.01, $\beta$: (0.9, 0.95)) with an effective batch size of 16 across 8 NVIDIA H100 GPUs. 

We process a 480-hour subset of the OpenDV-YouTube (ODV-YT) dataset \cite{opendv}, segmented into 1-minute clips at 10 fps. Camera poses are reconstructed using $\pi_3$ \cite{pi3} every 0.5 seconds, with intermediate frames recovered via spherical linear and linear interpolation. Semantic scene descriptions are generated by Qwen2.5-Omni \cite{qwen2_5omni} from frames sampled at 20-second intervals.

For keyframe generation, the model is trained with $|K|=21$ and temporal strides $\Delta k \in \{4, 8, 16\}$, while $\Delta k_{\text{test}}=8$ is fixed during inference. For the interpolation stage, the number of conditioning keyframes $|K|$ is randomly sampled from $[1, 10]$ during training. In the inference phase, we apply Motion-Inductive Noisy Conditioning with $\alpha_c=0.7$, $\sigma_c=0.3$, and set the boundary overlap to $p=1$. Please refer to the supplementary material for further details.

\noindent\textbf{Baselines.} We compare our method with Diffusion Forcing (DiffF)~\cite{df1}, Self Forcing(SelfF)~\cite{selfforcing}, Deep Forcing(DeepF)~\cite{deep_forcing}, Vista~\cite{vista} and Stable Virtual Camera (SEVA)~\cite{seva}. All models are finetuned for fair comparison. 

\noindent\textbf{Datasets.} To evaluate the capability of generating videos across large spatial spans and maintain structural integrity in expansive static scenes, we conduct experiments on two large-scale datasets: the ODV-YT validation set and the nuScenes~\cite{nuscenes2019} validation set for zero-shot testing. All experiments are conducted on 150 sequences.

\noindent\textbf{Evaluation Metrics.} For \textit{video quality}, we use Fréchet Inception Distance (FID)~\cite{fid} and Fréchet Video Distance (FVD)~\cite{fvd} as primary metrics, reporting both overall and per-segment scores to assess stability. For \textit{camera control accuracy}, we utilize $\pi_3$~\cite{pi3} to extract 6-DoF camera poses and report Absolute Trajectory Error (ATE) and Absolute Rotation Error (ARE) after alignment. For \textit{segment continuity}, we further use Motion Smoothness (MS)~\cite{huang2024vbench++}, PSNR and SSIM~\cite{ssim} to evaluate the discrepancy between adjacent generated segments.

\begin{table}[!t]
\centering
\caption{Quantitative comparison on 32-second videos of the ODV-YT dataset. Metrics are reported as \textbf{FID / FVD} for visual quality and temporal continuity, and \textbf{ATE / ARE} for global path adherence. Our method maintains superior stability and fidelity over extended durations.}
\label{tab:main_res}
\small
\resizebox{\linewidth}{!}{
\begin{tabular}{lccccccc}
\toprule
Method & Overall FID/FVD$\downarrow$ & 0-8 s & 8-16 s & 16-24 s & 24-32 s & ATE $\downarrow$ & ARE $\downarrow$ \\ 
\midrule
DiffF & 35.0 / 664.1 & 25.3 / 390.7 & 39.8 / 684.5 & 45.2 / 856.8 & 54.1 / 1067.1 & 0.469 & 19.448 \\ 
SelfF & 58.0 / 2113.6 & 26.2 / 533.5 & 62.8 / 1911.3 & 92.8 / 3402.2 & 119.4 / 4183.8 & 0.610 & 14.386 \\
DeepF & 42.3 / 1558.5 & 26.1 / 545.3 & 49.1 / 1472.3 & 63.7 / 2342.2 & 77.2 / 2832.6 & 0.571 & 15.144 \\
Vista & 66.7 / 1550.0 & 29.1 / 428.7 & 66.3 / 1212.6 & 103.9 / 2452.3 & 134.1 / 3774.5 & 0.641 & 19.332 \\
SEVA & 22.2 / 548.0 & 27.2 / 582.0 & 30.7 / 643.2 & 31.8 / 623.8 & 33.1 / 602.0 & 0.294 & 8.527 \\ 
\textbf{Ours} & \textbf{19.2 / 203.7} & \textbf{19.6 / 191.4} & \textbf{25.1 / 255.8} & \textbf{27.3 / 276.0} & \textbf{28.6 / 313.8} & \textbf{0.237} & \textbf{7.669} \\ 
\bottomrule
\end{tabular}
}
\end{table}

\begin{table}[!h]
  \centering
  \caption{Quantitative evaluation of zero-shot generalization on 16-second videos of the nuScenes dataset. Our approach achieves consistent high-fidelity generation and precise camera control in both FVD and ATE.}
  \label{tab:main_res_nuscenes}
  \small
  \resizebox{\linewidth}{!}{
  \begin{tabular}{lccccccc}
    \toprule
      Method & Overall FID/FVD$\downarrow$ & 0-4 s & 4-8 s & 8-12 s & 12-16 s & ATE $\downarrow$ & ARE $\downarrow$ \\ 
      \midrule
      DiffF  & 37.1 / 566.6 & 29.3 / 376.3 & 41.8 / 573.7 & 51.7 / 738.9 & 60.9 / 906.2 & 0.154 & 13.704 \\
      SelfF & 53.1 / 711.2 & 53.1 / 394.8 & 59.3 / 592.1 & 69.1 / 989.0 & 78.7 / 1504.7 & 0.092 & 12.464 \\
      DeepF & 50.3 / 618.4 & 52.5 / 343.3 & 59.0 / 577.3 & 64.7 / 887.2 & 69.9 / 1174.8 & 0.104 & 10.774 \\
      Vista & 68.7 / 1239.1 & 27.3 / 445.1 & 54.2 / 847.0 & 111.8 / 1860.8 & 160.5 / 2966.4 & 0.195 & 16.607 \\
      SEVA  & 35.9 / 487.9 & 30.2 / 493.7 & 41.3 / 529.1 & 54.7 / 595.2 & 72.1 / 769.4 & 0.117 & 6.289 \\
      \textbf{Ours}  & \textbf{19.6 / 225.4} & \textbf{23.3 / 239.9} & \textbf{26.5 / 264.3} & \textbf{31.1 / 323.0} & \textbf{30.8 / 345.2} & \textbf{0.045} & \textbf{5.274}  \\
      \bottomrule
  \end{tabular}
  }
\end{table}

\begin{figure}[h]
  \centering
   \includegraphics[width=0.95\linewidth]{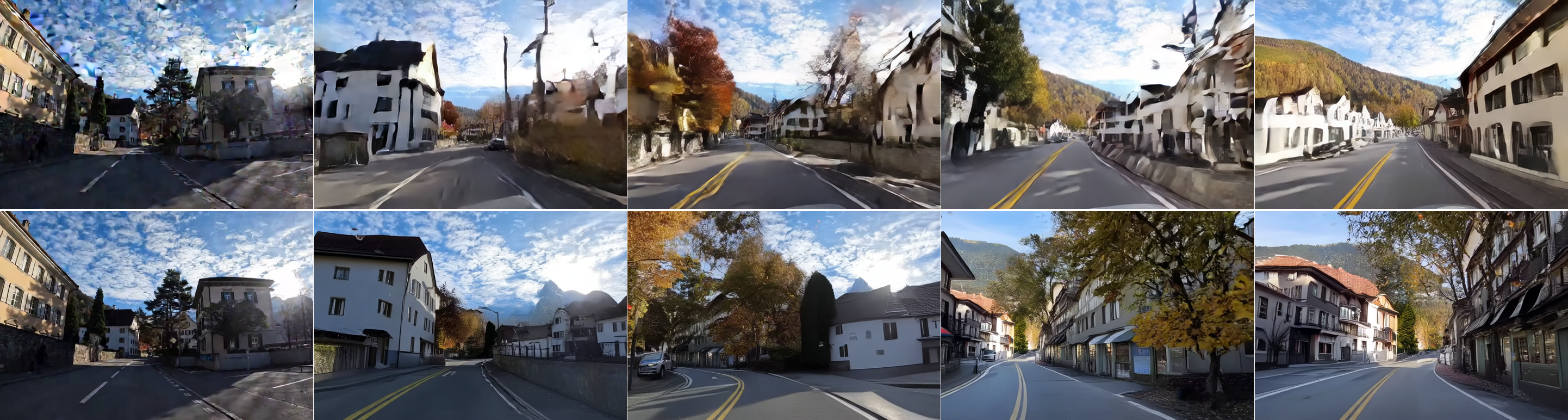}
    \caption{Qualitative comparison of keyframe model designs. The top row shows the results of the \textit{temporal compression} baseline, while the bottom row shows our proposed design. Our approach provides significantly clearer visual details and superior structural accuracy compared to the baseline.}
   \label{fig:ablation_keyframe}
\end{figure}

\subsection{Comparison on Long Trajectory Generation.} 
In \cref{fig:main_qual_res}, \cref{tab:main_res}, and \cref{tab:main_res_nuscenes}, we present qualitative and quantitative comparisons. While autoregressive baseline methods can generate reasonable short-term results, they suffer from severe error accumulation and temporal drifting as the video length increases. Notably, even with the incorporation of specialized strategies designed to mitigate drifting, these methods still exhibit sharp performance degradation on extended sequences. In contrast, our divide-and-conquer autoregressive design leverages global keyframes as structural anchors to ensure both long-range visual stability and effective camera control. This advantage is further demonstrated by our robust zero-shot generalization to unseen environments. When evaluated on the nuScenes dataset, our approach consistently maintains high generation quality and precise path adherence, significantly surpassing baseline methods without requiring additional fine-tuning. Besides, our method outperform SEVA in terms of camera controllability and video quality. Quantitatively, these results confirm that our framework remains remarkably stable across the entire horizon, providing superior performance in both in-distribution and out-of-distribution scenarios. For more visualization results, please refer to the supplementary material.

\subsection{Ablation Study}

\noindent\textbf{Keyframe Spatial-Structural Preservation.} We compare our image-level encoding with a temporal compression baseline to evaluate its necessity for sparse keyframe generation. Results are calculated using the generated keyframe sequences for 16-second videos on ODV-YT. As shown in \cref{tab:abl_encoding_strategy} and \cref{fig:ablation_keyframe}, applying temporal compression to keyframes with large spatial spans leads to significant performance drops. Specifically, the spatial displacement between distant keyframes makes the VAE's 4:1 temporal compression inherently lossy, which severely degrades quality (higher FID). Furthermore, such compression obscures fine-grained geometric cues, hindering the camera encoder from learning precise representations, as evidenced by the elevated ATE and ARE.
\begin{table}[h]
\small
  \caption{Ablation study on keyframe model design. We compare our proposed image-level encoding against a \textit{temporal compression} baseline, which treats keyframes as a video sequence with temporal downsampling. Our design achieves superior performance in both visual quality and trajectory adherence in keyframe generation.}
  \centering
  \begin{tabular}{@{}lccc@{}}
    \toprule
    Model Design & FID$\downarrow$ & ATE$\downarrow$ & ARE$\downarrow$ \\
    \midrule
    Temporal compression & 19.6 & 0.155 & 7.276 \\
    Ours & \textbf{16.3} & \textbf{0.100} & \textbf{3.999} \\
    \bottomrule
  \end{tabular}
  \label{tab:abl_encoding_strategy}
\end{table}

\begin{table}[h]
  \centering
  \small
  \caption{Ablation on keyframe intervals for generation ($\Delta k_{gen}$) and interpolation ($\Delta k_{int}$). Keyframes are initially generated at stride $\Delta k_{gen}$ and subsequently filtered to $\Delta k_{int}$ for dense interpolation.}
  \label{tab:keyframe_ablation_transposed}
  \begin{tabular}{@{}lcccccc@{}}
    \toprule
    $(\Delta k_{gen}, \Delta k_{int})$ & (4, 4) & (4, 8) & (4, 16) & \textbf{(8, 8)} & (8, 16) & (16, 16) \\
    \midrule
    FID $\downarrow$ & 21.4 & 21.0 & 20.9 & 19.3 & 19.0 & \textbf{18.8} \\
    FVD $\downarrow$ & 201.8 & 199.9 & 204.6 & 196.0 & 201.5 & \textbf{187.8} \\
    ATE $\downarrow$ & 0.148 & 0.147 & 0.131 & \textbf{0.096} & 0.098 & 0.104 \\
    ARE $\downarrow$ & 4.98 & 4.91 & 5.08 & \textbf{3.96} & 4.06 & 4.74 \\
    \bottomrule
  \end{tabular}
\end{table}

\noindent\textbf{Keyframe Sampling Strides.} 
We investigate the impact of sampling strategies by varying the generation stride $\Delta k_{gen}$ and interpolation stride $\Delta k_{int}$ on 16-second ODV-YT video generation. Specifically, keyframes are initially synthesized at a stride of $\Delta k_{gen}$ and subsequently filtered to the interpolation stride $\Delta k_{int}$ to serve as conditioning anchors for the dense synthesis stage. As shown in \cref{tab:keyframe_ablation_transposed}, $\Delta k_{gen}$ is the primary determinant of both visual quality and trajectory fidelity. Increasing $\Delta k_{gen}$ from 4 to 16 consistently enhances the quality of structural anchors, which directly improves the final FID and FVD. In contrast, the interpolation stride $\Delta k_{int}$ exerts a negligible influence on the overall performance. Regarding trajectory adherence, we find that neither excessively dense nor sparse generation is optimal for ATE and ARE; the former introduces redundant constraints that hinder motion naturalness, while the latter provides insufficient guidance to bound drift. Our balanced configuration of $(8, 8)$ achieves the optimal trade-off between structural stability and generative quality.

\begin{figure}[!t]
  \centering
   \includegraphics[width=0.99\linewidth]{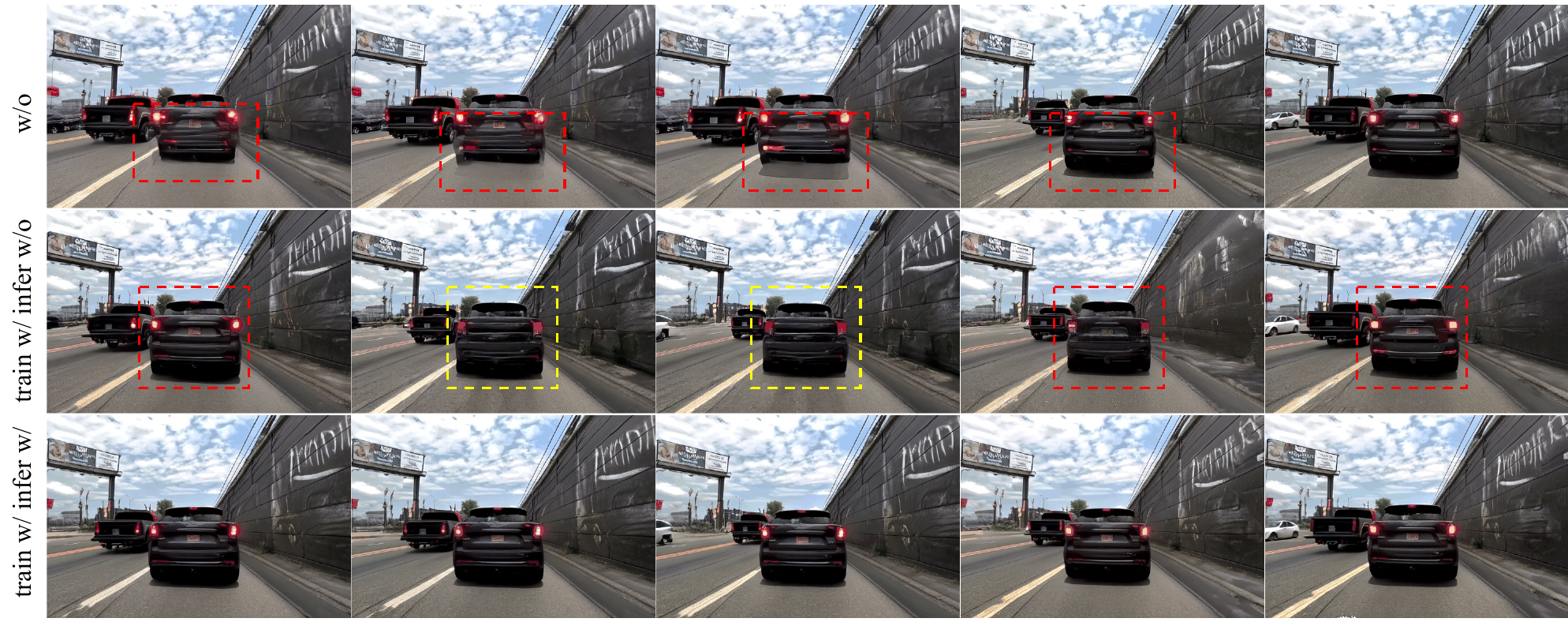}

    \caption{Ablation on noisy keyframe design. We compare \textit{W/o noisy keyframe}, \textit{Train w/ noisy keyframe}, and our full \textit{Train/infer w/ noisy keyframe} strategy. Applying noise during both stages effectively mitigates memory bias and ensures smoother transitions.}
    
   \label{fig:Ablation_Interpolation_Memnoise}
\end{figure}

\begin{table}[!t]
  \small
  \caption{Ablation on noisy keyframe design. All methods take the same generated keyframes as input. Applying noisy keyframes during both training and inference achieves the best dynamic video quality (lowest FVD).}
  \centering
  \begin{tabular}{@{}lcc@{}}
    \toprule
    Training Design & FID$\downarrow$ & FVD$\downarrow$  \\
    \midrule
    W/o noisy keyframe &  \textbf{19.0} & 246.5 \\
    Train w/ noisy keyframe &  19.9 & 240.9 \\
    Train/infer w/ noisy keyframe & 19.3 & \textbf{201.7} \\
    \bottomrule
  \end{tabular}
  \label{tab:abl_memnoise}
\end{table}
\label{sec:memnoise_design}
\noindent\textbf{Motion-Inductive Noisy Conditioning.} 
Conditioning the interpolation model on noisy keyframes is essential for high-quality synthesis, as discussed in \cref{sec:noisy_keyframe}. We evaluate this in \cref{tab:abl_memnoise} and \cref{fig:Ablation_Interpolation_Memnoise} on 16-second ODV-YT sequences. Quantitatively, we observe that while FID remains relatively stable, the FVD significantly deteriorates when noise is not consistently applied during both training and inference. Qualitative results in \cref{fig:Ablation_Interpolation_Memnoise} reveal two distinct artifact patterns: (1) \textit{unnatural transitions}, such as fade-like blending between keyframes, and (2) \textit{motion stagnation}, where the model replicates imperfect conditioning keyframes. These temporal artifacts—which involve mode collapse where the model copies pixels rather than learning motion dynamics—explain why the FVD increases sharply despite the stable FID. This confirms that joint noisy training and inference forces the model to bypass "copy-paste" shortcuts and achieve smoother, more realistic video generation.

\begin{table}[t]
  \small
  \caption{Ablation on segment continuity design. We compare \textit{Overlap} (simple overlapping), \textit{+Sub} (latent substitution), and our full \textit{+Sub+Train} strategy. Our method effectively bridges segment junctions, achieving the best motion smoothness and structural fidelity.}
  \centering
  \setlength{\tabcolsep}{4pt} 
  \begin{tabular}{@{}lccc@{}}
    \toprule
    Design & MS$\uparrow$ & PSNR$\uparrow$ & SSIM$\uparrow$ \\
    \midrule
    Overlap & 0.969 & 18.15 & 0.534  \\
    + Sub & 0.968 & 17.95 & 0.524  \\
    + Sub + Train & \textbf{0.972} & \textbf{18.75} & \textbf{0.562} \\ 
    \bottomrule
  \end{tabular}
  \label{tab:abl_continuity}
\end{table}
\begin{figure}[t]
  \centering
   \includegraphics[width=0.95\linewidth]{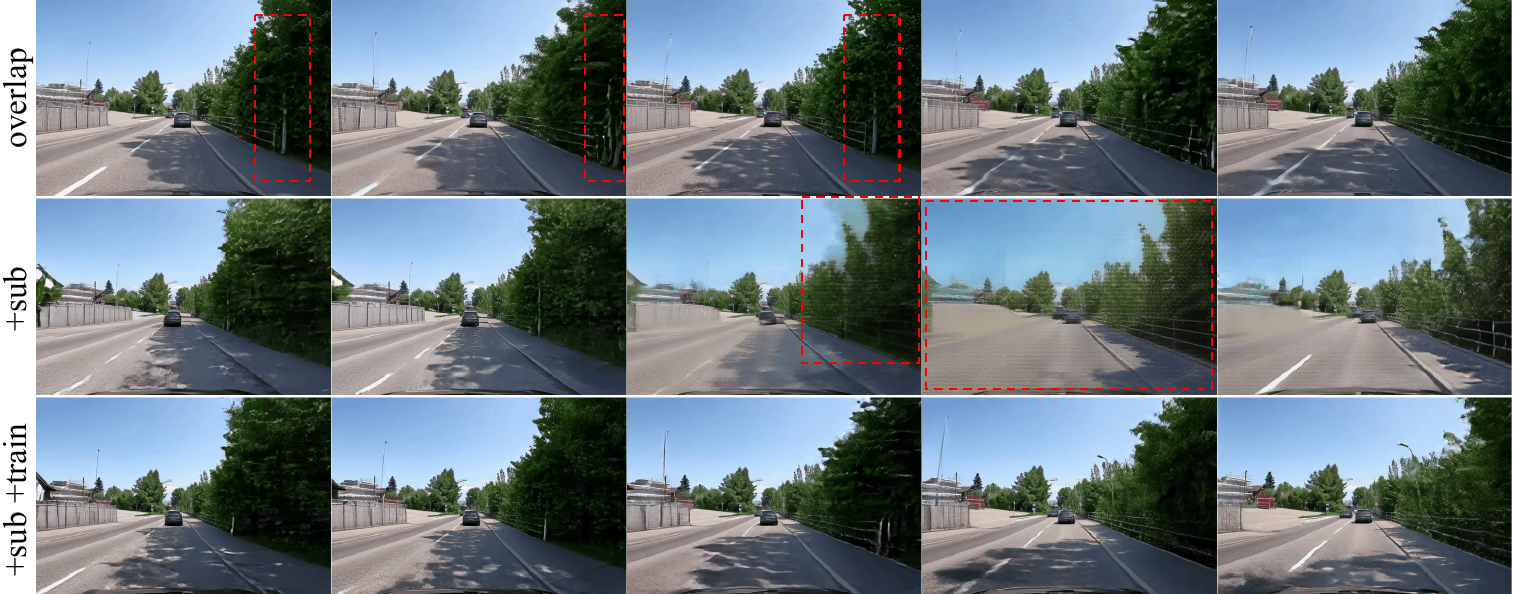}
    \caption{Ablation on Frame Continuity Design. The designs test continuity at segment boundaries: \textbf{overlap} (simple overlap); \textbf{+ sub} (overlap + latent substitution); and \textbf{+ sub + train} (our trained continuity mechanism).}
   \label{fig:Ablation_Interpolation_Continuity}
\end{figure}

\noindent\textbf{Seamless Boundary Consistency.}
We evaluate the segment consistency mechanism by providing Ground-Truth (GT) frames as boundary conditions and comparing generated frames proximal to the segment junctions against these targets. As shown in \cref{tab:abl_continuity} and \cref{fig:Ablation_Interpolation_Continuity}, simple segment overlap (\textit{Overlap}) fails to guarantee visual coherence; the final frame of a preceding segment often deviates from the required starting pose of the subsequent one, resulting in noticeable motion and temporal fractures. While naively replacing the first frame with the preceding segment's last latent (\textit{+Sub}) ensures junction continuity, it introduces pronounced artifacts. This stems from a distribution mismatch where the model, expecting noisy inputs, struggles to process a clean and deterministic conditioning frame without explicit consistency training. In contrast, our full strategy (\textit{+Sub+Train}) effectively bridges this gap, achieving superior smoothness and structural fidelity at segment junctions.

\begin{table}[t]
\centering
\caption{Ablation study on the effect of keyframes for 32-second videos on the OpenDV-YT dataset. The inclusion of keyframes significantly enhances long-term consistency and trajectory fidelity.}
\label{tab:keyframe_ablation_32s}
\small
\resizebox{\linewidth}{!}{
\begin{tabular}{lccccccc}
\toprule
Method & Overall FID/FVD$\downarrow$ & 0-8 s & 8-16 s & 16-24 s & 24-32 s & ATE $\downarrow$ & ARE $\downarrow$ \\ 
\midrule
w/o Keyframe & 25.2 / 376.7 & 21.1 / 247.1 & 31.5 / 407.5 & 38.2 / 546.5 & 45.2 / 665.3 & 0.387 & 12.184 \\ 
\textbf{Ours} & \textbf{19.2 / 203.7} & \textbf{19.6 / 191.4} & \textbf{25.1 / 255.8} & \textbf{27.3 / 276.0} & \textbf{28.6 / 313.8} & \textbf{0.237} & \textbf{7.669} \\ 
\bottomrule
\end{tabular}
}
\end{table}
\begin{figure}[!h]
  \centering
   \includegraphics[width=0.95\linewidth]{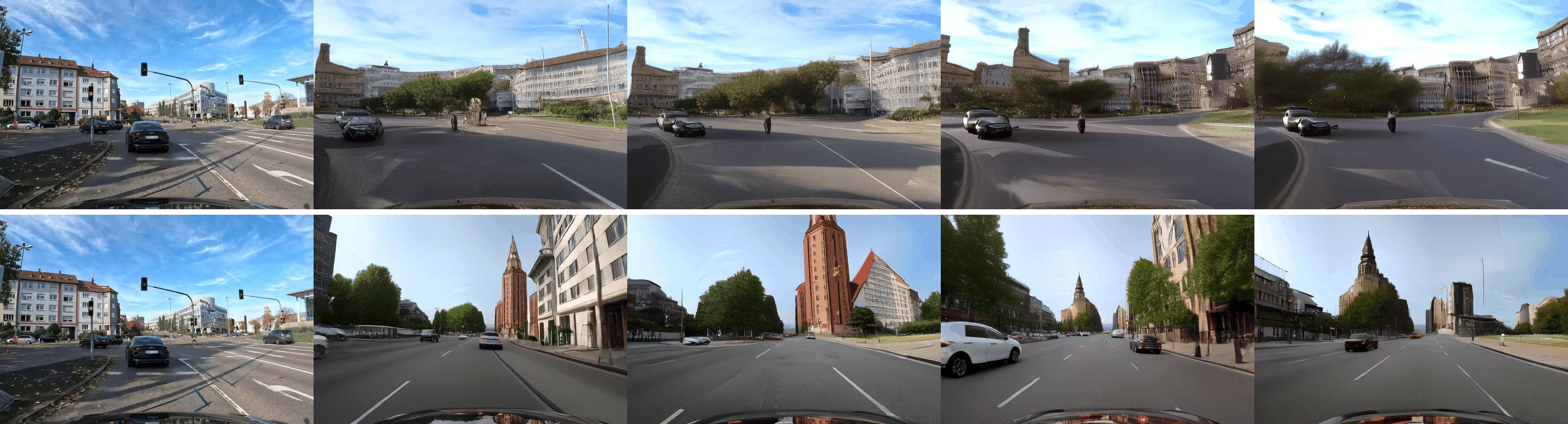}
    \caption{Qualitative ablation of keyframe guidance on 32-second videos. The top row shows the results without keyframe conditioning, and the bottom row shows our full model with keyframe anchors.}
   \label{fig:Ablation_keyframe_usage}
\end{figure}

\noindent\textbf{Necessity of Keyframe Anchoring.} 
We evaluate the necessity of structural anchors by removing keyframe conditioning during the interpolation stage. Quantitatively, \cref{tab:keyframe_ablation_32s} shows that without keyframe guidance, the model suffers from long-term drifting and poor camera accuracy, with FID/FVD increasing sharply as the sequence extends and ATE/ARE being significantly worse. This confirms that keyframes are essential for bounding accumulated errors over long durations. Qualitatively, \cref{fig:Ablation_keyframe_usage} illustrates that the unguided variant fails to maintain camera control during complex maneuvers, such as a left turn followed by a straight trajectory, resulting in noticeable scene distortion and path deviation. In contrast, our full model leverages keyframe anchors to ensure both visual stability and precise trajectory adherence across the entire 32-second horizon.
\section{Conclusion and Future Works}

In this paper, we presented \OurName{}, a divide-and-conquer framework that reformulates long-trajectory video generation through an autoregressive modeling paradigm. By decoupling the complex synthesis task into global structural planning via keyframe anchors and local refinement via segment-wise interpolation, our approach effectively bounds the cumulative drift inherent in traditional autoregressive methods. Extensive experiments on the large-scale OpenDV-YouTube and nuScenes datasets demonstrate that \OurName{} consistently achieves state-of-the-art performance in both visual fidelity and trajectory adherence. Notably, our framework enables the stable synthesis of high-quality long trajectories for extended horizons up to 32 seconds, significantly surpassing existing baselines.  

While our dual-generator design ensures stability, a promising direction is to investigate unified architectures that can jointly model the distribution of sparse keyframes and dense sequences within a single latent space without compromising generation quality.

\section*{Acknowledgements}
Research was sponsored by the Army Research Office and was accomplished under Cooperative Agreement Number W911NF-25-2-0040. The views and conclusions contained in this document are those of the authors and should not be interpreted as representing the official policies, either expressed or implied, of the Army Research Office or the U.S. Government. The U.S. Government is authorized to reproduce and distribute reprints for Government purposes notwithstanding any copyright notation herein.

This material is based upon work supported by the Intelligence Advanced Research Projects Activity under prime Contract No. 140D0423C0034.

\bibliographystyle{splncs04}
\bibliography{main}

\clearpage
\appendix
\noindent{\large \textbf{Appendix}}
\\

\setcounter{section}{0}
\setcounter{figure}{0}
\setcounter{table}{0}
\renewcommand{\thesection}{\Alph{section}}
\renewcommand{\thefigure}{\Alph{figure}}
\renewcommand{\thetable}{\Alph{table}}

\setcounter{page}{1}

\renewcommand\thesection{\Alph{section}}
\renewcommand\thesubsection{\thesection.\arabic{subsection}}
\renewcommand\thesubsubsection{\thesubsection.\arabic{subsubsection}}
\setcounter{section}{0}

\noindent\textbf{Overview.} In this supplementary material, we first provide a detailed proof of Proposition 1 from the main paper in Section~\ref{sec:supp:proof}. We then present additional implementation details in Section~\ref{sec:supp:implement_details}, followed by additional results and analysis in Section~\ref{sec:supp:results}. Finally, we discuss the limitations of our work in Section~\ref{sec:supp:limit}.

\section{Proof of Proposition 1}
\label{sec:supp:proof}

We evaluate the generation quality using the Mean Squared Error (MSE), defined as $\text{MSE}(e_t) = \mathbb{E}[\|e_t\|^2] = \|\mathbb{E}[e_t]\|^2 + \text{Tr}(\text{Cov}(e_t))$, where $\|\mathbb{E}[e_t]\|^2$ represents the squared systematic bias (drift) and $\text{Tr}(\text{Cov}(e_t))$ represents the error variance.

\subsection*{Part I: Statistical Divergence of Pure Autoregressive Generation}

In pure AR generation, $I_{t+1} = \mathcal{F}(I_t, \epsilon_t)$, where $I_t$ is the generated frame and $\epsilon_t$ is the sampling noise. We analyze the stability by considering the cumulative error $e_t = I_t - I_t^*$, where $I_t^*$ is the ground-truth frame. Using a first-order perturbation, we add and subtract $\mathcal{F}(I_t^*, \epsilon_t)$ to decompose the error:
\begin{align}
    e_{t+1} &= [\mathcal{F}(I_t, \epsilon_t) - \mathcal{F}(I_t^*, \epsilon_t)] + [\mathcal{F}(I_t^*, \epsilon_t) - I_{t+1}^*] \nonumber \\
    &\approx \nabla_{I_t^*} \mathcal{F} \cdot e_t + \eta_{t+1},
\end{align}
where $\eta_{t+1}$ is the single-step drift when conditioned on ground truth. Assuming $\mathcal{F}$ is $L$-Lipschitz with $L = \sup \|\nabla_{I_t^*} \mathcal{F}\|$, we obtain the recursive norm bound $\|e_{t+1}\| \le L \|e_t\| + \|\eta_{t+1}\|$. 

\begin{proof}[Derivation of AR Divergence]

\noindent\textbf{Divergence Upper Bound.} Unrolling the recursive inequality $\|e_t\| \le L \|e_{t-1}\| + \|\eta_t\|$ from $t=1$ to $N$, with $\|e_0\|=0$, we have:
\begin{align}
    \|e_1\| &\le \|\eta_1\| \nonumber \\
    \|e_2\| &\le L\|e_1\| + \|\eta_2\| \le L\|\eta_1\| + \|\eta_2\| \nonumber \\
    &\dots \nonumber \\
    \|e_N^{\text{base}}\| &\le \sum_{t=1}^N L^{N-t} \|\eta_t\|.
\end{align}
This formula reveals the catastrophic accumulation of error in pure AR models:
\begin{itemize}
    \item \textbf{Case $L=1$:} Assuming $\|\eta_j\| \le \eta$, the bound simplifies to $\|e_N\| \le \sum_{j=1}^N \eta = \mathcal{O}(N)$, representing a linear growth of error norm bound.
    \item \textbf{Case $L>1$:} The geometric series diverges as $\mathcal{O}(L^N)$, where local artifacts are exponentially amplified.
\end{itemize}

\noindent\textbf{Divergence Lower Bound.} While the Lipschitz upper bound describes the maximum potential error, the intrinsic divergence of pure AR generation is governed by the persistent nature of single-step errors $\eta_t$. In practice, the learned mapping $\mathcal{F}$ is subject to \textit{exposure bias}, where the distribution gap between training and inference leads to a non-vanishing systematic drift. We formally assume that the expected single-step error is bounded away from zero, i.e., $\|\mathbb{E}[\eta_t]\| \ge \mu > 0$, or potentially increases with $t$ as the state $I_t$ departs from the training distribution.

Furthermore, to prevent visual signal decay over time, the mapping $\mathcal{F}$ must be non-contractive, meaning it does not shrink the latent state space in any direction. Formally, we assume the minimum singular value (the smallest scaling factor) of the Jacobian $\mathbf{J}_t = \nabla_{I_t^*} \mathcal{F}$ satisfies $\sigma_{\min}(\mathbf{J}_t) \ge 1$ for all $t$. Under this non-contractive assumption, the expected cumulative error $e_N^{base}$ is unrolled through the product of time-varying Jacobians:
\begin{equation}
    \|\mathbb{E}[e_N^{\text{base}}]\| = \left\| \sum_{t=1}^N \left( \prod_{k=t}^{N-1} \mathbf{J}_k \right) \mathbb{E}[\eta_t] \right\| \ge \sum_{t=1}^N \left( \prod_{k=t}^{N-1} \sigma_{\min}(\mathbf{J}_k) \right) \mu \ge N\mu,
\end{equation}
where $\prod_{k=t}^{N-1} \mathbf{J}_k$ represents the error transition from step $t$ to $N$. Since the growth is at least linear in $N$, the squared bias satisfies $\|\mathbb{E}[e_N^{\text{base}}]\|^2 = \Omega(N^2)$.

\noindent\textbf{Variance Accumulation.} Even in an idealized scenario where the generator is unbiased ($\mathbb{E}[\eta_t] = 0$), the AR process remains unstable due to stochastic noise accumulation. Modeling $\{\eta_t\}$ as i.i.d. random variables with variance $\text{Tr}(\text{Cov}(\eta_t)) = \sigma^2$, the variance of the cumulative error is:
\begin{equation}
    \text{Tr}(\text{Cov}(e_N^{\text{base}})) = \sum_{t=1}^N \text{Tr}(\text{Cov}(\eta_t)) = N\sigma^2.
\end{equation}
Under the decomposition $\text{MSE} = \|\text{Bias}\|^2 + \text{Variance}$, the Mean Squared Error is lower-bounded by $\Omega(N)$. Physically, this indicates that pure AR models lack a corrective mechanism; without global anchors, both systematic drift and stochastic noise force the trajectory to diverge irreversibly from the ground truth.
\end{proof}

\subsection*{Part II: Error Bounding via Keyframe-Assisted Interpolation}

In contrast to the recursive accumulation in pure AR models, our strategy partitions the sequence into segments $S_i$ using sparse keyframes $\{K_i\}$ at interval $T$. For any frame $t = iT + \tau$ where $\tau \in (0, T)$ is the temporal offset within segment $S_i$, the generation process is defined by a conditional interpolation mapping $\mathcal{G}$:
\begin{equation}
    I_{iT+\tau} = \mathcal{G}(K_i, K_{i+1}, S_{i-1}, \tau, \epsilon_\tau),
\end{equation}
where $K_i$ and $K_{i+1}$ are the bidirectional global anchors, $S_{i-1}$ provides the preceding temporal context (momentum), and $\epsilon_\tau$ represents the local sampling noise.

\subsubsection*{Step 1: Unified Decomposition Framework}

For any frame $t = iT + \tau$ in segment $S_i$ ($\tau \in (0, T)$), we model the generation process as the superposition of two components: a \textbf{Brownian Bridge} $B_\tau$ pinned at the keyframes $\{K_i, K_{i+1}\}$ to ensure global structural alignment, and an \textbf{autoregressive residual} $\Phi_{i}(\tau)$ to maintain temporal continuity from the preceding segment $S_{i-1}$. Specifically, $I_{iT+\tau} = B_\tau + \Phi_{i}(\tau)$, where the Brownian Bridge $B_\tau$ is a stochastic process characterized by:
\begin{align}
    \mathbb{E}[B_\tau] &= \left(1 - \frac{\tau}{T}\right) K_i + \frac{\tau}{T} K_{i+1}, \\
    \text{Var}(B_\tau) &= \frac{\tau(T-\tau)}{T} \sigma_{int}^2,
\end{align}
where $\sigma_{int}^2$ denotes the interpolation noise intensity. Under this framework, the generated state $I_{iT+\tau}$ is the sum of the bridge mean, the cross-segment consistency shift $\Phi_{i}(\tau)$ (representing momentum), and stochastic bridge noise $\epsilon_\tau \sim \mathcal{N}(0, \text{Var}(B_\tau))$. 

Assuming the ground-truth trajectory $I_t^*$ follows a similar Brownian Bridge dynamic anchored at ground-truth keyframes $\{K_i^*\}$, the interpolation error $e_{iT+\tau}^{\text{ours}} = I_t - I_t^*$ is derived by subtracting the respective components:
\begin{align}
    \label{eq:bridge_sum}
    e_{iT+\tau}^{\text{ours}} = \underbrace{\left(1 - \frac{\tau}{T}\right) e(K_i) + \frac{\tau}{T} e(K_{i+1})}_{E_{anch} \text{ (Boundary Interpolation)}} + \underbrace{\delta_{i}(\tau)}_{E_{leak} \text{ (Momentum Leakage)}} + \underbrace{w_\tau}_{E_{noise} \text{ (Local Noise)}},
\end{align}

where $E_{anch}$ results from the linear interpolation of keyframe drifts $e(K_i) = K_i - K_i^*$, $E_{leak} = \Phi - \Phi^*$ is the implicit autoregressive error (momentum leakage), and $E_{noise} = \epsilon - \epsilon^*$ is the local stochastic hallucination. To prove global stability, we must show that $E_{leak}$ remains bounded independently of the sequence length $N$.

\subsubsection*{Step 2: Damping of Momentum Leakage}

To maintain temporal continuity while adhering to fixed keyframes, the video model implicitly minimizes the acceleration along the error trajectory $\delta_{i}(\tau)$, where $\tau \in [0, T]$ is the frame offset. This corresponds to minimizing the second-order energy functional:
\begin{equation}
    L(\delta_{i}) = \int_{0}^{T} (\delta_{i}''(\tau))^2 d\tau.
\end{equation}
Under these dynamics, the momentum error entering from the preceding segment $S_{i-1}$ is defined as $\Delta v_{i-1} = \delta_{i-1}'(T)$. This term acts as the initial velocity constraint $\delta_{i}'(0) = \Delta v_{i-1}$, which is then structurally damped across the current segment.

\begin{lemma}[Damping Coefficient]
\label{lemma:damping}
The transmitted velocity error $\Delta v_i$ at the boundary satisfies the recurrence $\Delta v_i = \gamma \Delta v_{i-1}$, where the damping factor is:
\begin{equation}
    \gamma = -{1\over2}.
\end{equation}
\end{lemma}

\begin{proof}
The Euler-Lagrange equation for a functional involving second-order derivatives, $L = \int f(\tau, \delta_i, \delta_i', \delta_i'') d\tau$ with $f = (\delta_i'')^2$, is given by:
\begin{equation}
    \frac{\partial f}{\partial \delta_i} - \frac{d}{d\tau} \left( \frac{\partial f}{\partial \delta_i'} \right) + \frac{d^2}{d\tau^2} \left( \frac{\partial f}{\partial \delta_i''} \right) = 0.
\end{equation}
Since $f$ only depends on the second derivative $\delta_i''$, we have $\frac{\partial f}{\partial \delta_i} = 0$ and $\frac{\partial f}{\partial \delta_i'} = 0$. Substituting $\frac{\partial f}{\partial \delta_i''} = 2\delta_i''(\tau)$ into the equation yields:
\begin{equation}
    \frac{d^2}{d\tau^2} (2\delta_i''(\tau)) = 0 \implies 2\delta_i^{(4)}(\tau) = 0
\end{equation}
Thus, the governing equation is $\delta_i^{(4)}(\tau) = 0$. The general solution to this fourth-order linear differential equation is a cubic polynomial:
\begin{equation}
    \delta_i(\tau) = a\tau^3 + b\tau^2 + c\tau + d.
\end{equation}
To isolate the effect of velocity leakage, we apply the positional anchoring and smoothness constraints:
\begin{align}
    \delta_{i}(0) = 0, \quad \delta_{i}(T) = 0 & \quad \text{(Rigid Positional Anchors)}, \\
    \delta_{i}'(0) = \Delta v_{i-1} & \quad \text{(Velocity Continuity)}, \\
    \delta_{i}''(T) = 0 & \quad \text{(Natural Boundary Condition).}
\end{align}
Substituting these conditions into the cubic form yields the following algebraic system:
\begin{align}
    d = 0, \quad c &= \Delta v_{i-1}, \\
    6aT + 2b &= 0 \implies b = -3aT, \\
    aT^3 + (-3aT)T^2 + \Delta v_{i-1}T &= 0 \implies a = \frac{\Delta v_{i-1}}{2T^2}.
\end{align}
The velocity error transmitted to the start of the next segment $S_{i+1}$ is the derivative at the current segment's end $\tau=T$:
\begin{align}
    \Delta v_i = \delta_{i}'(T) &= 3aT^2 + 2bT + c \nonumber \\
    &= 3\left(\frac{\Delta v_{i-1}}{2T^2}\right)T^2 + 2\left(-\frac{3\Delta v_{i-1}}{2T}\right)T + \Delta v_{i-1} \nonumber \\
    &= \frac{3}{2}\Delta v_{i-1} - 3\Delta v_{i-1} + \Delta v_{i-1} \nonumber \\
    &= -\frac{1}{2} \Delta v_{i-1}.
\end{align}
\end{proof}
The geometric progression $|\gamma| = 0.5 < 1$ confirms that the momentum error is exponentially attenuated across segments, ensuring that the system remains globally stable and avoids the cumulative divergence typical of pure AR models. 

Furthermore, we analyze the maximum magnitude of the leakage term $\|\delta_{i}(\tau)\|$ within a single segment to bound the local drift. By identifying the extremum where $\delta_i'(\tau) = 0$ (occurring at $\tau^* = T(1 - \frac{\sqrt{3}}{3})$), the intra-segment peak error is given by:
\begin{equation}
    \max_{\tau \in [0, T]} \|\delta_{i}(\tau)\| = \frac{\sqrt{3}}{9} T \|\Delta v_{i-1}\| \approx 0.192 T \|\Delta v_{i-1}\|.
\end{equation}
This result shows that the momentum leakage is locally constrained by the keyframe interval $T$ and remains strictly proportional to the incoming velocity error, providing a theoretical guarantee for the unified global bound.

\subsubsection*{Step 3: Synthesis of the Unified Global Bound}

Substituting the results from the previous steps into the decomposition framework of \cref{eq:bridge_sum}, we synthesize the final upper bound for the generation error. By the triangle inequality, the total error norm is bounded by the sum of three distinct components:
\begin{equation}
    \|e_{iT+\tau}^{\text{ours}}\| \le \underbrace{\|E_{anch}\|}_{\text{Keyframe Drift}} + \underbrace{\|E_{leak}\|}_{\text{Momentum Leakage}} + \underbrace{\|E_{noise}\|}_{\text{Interpolation Variance}}.
\end{equation}

We analyze each component to establish the global stability:

\begin{itemize}
    \item \textbf{Keyframe Drift ($\|E_{anch}\|$):} This term depends on the strategy used to generate $\{K_j\}$. 
    \begin{enumerate}
        \item \textit{Global Generation:} If keyframes are generated globally (e.g., via a sparse pre-selected set), $\|e(K_j)\| \le C_{kf}$, yielding $\|E_{anch}\| = \mathcal{O}(1)$.
        \item \textit{Downsampled AR:} If keyframes are generated autoregressively with step $T$, the error accumulates over $N/T$ steps. Following Part I, $\|e(K_j)\| = \mathcal{O}(N/T)$, which represents a $T$-fold suppression of the pure AR drift $\mathcal{O}(N)$.
    \end{enumerate}

    \item \textbf{Momentum Leakage ($\|E_{leak}\|$):} From \cref{lemma:damping}, the velocity error satisfies $\Delta v_i = \gamma \Delta v_{i-1}$ with $|\gamma| = 0.5$. The cumulative velocity error is bounded by the geometric series $\sum_{k=0}^i \gamma^k \|\Delta v_0\| \le 2\|\Delta v_0\|$. Combined with the intra-segment peak derived in Step 2:
    \begin{equation}
        \sup_{i, \tau} \|\delta_i(\tau)\| \le 0.192 T \cdot (2\|\Delta v_0\|) = 0.384 T \|\Delta v_0\|.
    \end{equation}
    Crucially, this term is only dependent of the keyframe interval $(T)$ and independent of the total sequence length $N$, confirming that momentum error does not diverge.

    \item \textbf{Interpolation Variance ($\|E_{noise}\|$):} The local stochastic error is governed by the Brownian Bridge variance $\text{Var}(B_\tau) = \frac{\tau(T-\tau)}{T} \sigma_{int}^2$. To find the upper bound, we maximize this quadratic form with respect to $\tau$:
    \begin{equation}
        \frac{d}{d\tau} \left( \frac{\tau T - \tau^2}{T} \right) = 0 \implies \tau = \frac{T}{2}; \quad \max_\tau \text{Var}(B_\tau) = \frac{T}{4} \sigma_{int}^2.
    \end{equation}
    Thus, $\|E_{noise}\| \le \frac{\sqrt{T}}{2} \sigma_{int}$, indicating that local hallucination is strictly constrained by the segment length.
\end{itemize}

\noindent \textbf{Conclusion:} 
Combining these results, the unified global bound is:
\begin{equation}
    \label{eq:unified_bound_final}
    \|e_t^{\text{ours}}\| \le \max_j \|e(K_j)\| + 0.384 T \|\Delta v_0\| + \frac{\sqrt{T}}{2} \sigma_{int}.
\end{equation}
Unlike the $\Omega(N)$ divergence of pure AR models, our method ensures that the error is dominated by the quality of sparse keyframes and the local segment length $T$. Even in the worst-case AR keyframe scenario, the MSE is suppressed by a factor of $T^2$, effectively neutralizing catastrophic temporal decay and establishing long-term generation stability.

\noindent \textbf{Conclusion on Keyframe Scenarios:} 
If keyframes are generated globally, $\|e_t^{\text{ours}}\| = \mathcal{O}(1)$. If keyframes are generated autoregressively with step $T$, $\|e_t^{\text{ours}}\| = \mathcal{O}(N/T)$, achieving a $T$-fold (and $T^2$ in MSE) suppression of degradation compared to pure AR. The theoretical results establish the advantage of our method over pure AR generation, and the subsequent experimental results further confirm this advantage empirically.

\section{Implementation Details}
\label{sec:supp:implement_details}
\subsection{Training Details}

\noindent\textbf{Interpolation Stage Training.} 
For the training of the interpolation generator $G_I$, we employ a stochastic keyframe conditioning strategy to enhance the model's robustness to diverse temporal contexts. Specifically, for each training segment, we randomly sample $1$ to $10$ conditioning keyframes from the keyframe set $\mathcal{K}$. These keyframes are restricted to a temporal window within $\pm 10$ frames of the target segment's range to simulate local structural anchors. Furthermore, to facilitate the \textit{Motion-Inductive Noisy Conditioning} (Sec. 4.3), the guidance coefficient $\alpha_c$ is randomly sampled from the range $[0.1, 0.5]$ during training, ensuring that the model learns to prioritize underlying motion dynamics over simple pixel-level replication.

\noindent\textbf{Baseline Implementation.} 
To ensure a fair and comprehensive comparison, we evaluate our model against several state-of-the-art baselines. For consistency, all compared models are provided with camera trajectories as conditional inputs. We independently implemented the training and inference pipelines for SEVA~\cite{seva} and Diffusion Forcing~\cite{df1} following the architectures described in their respective papers. For other baseline methods, we utilized their official open-source repositories and strictly adhered to the recommended training configurations and hyperparameters. All models, including our re-implementations, were trained until convergence to ensure they reached their optimal performance levels for the benchmarking tasks.

\noindent\textbf{Evaluation Metrics and Resolutions.} 
We report the Fréchet Inception Distance (FID) and Fréchet Video Distance (FVD) to assess visual quality and temporal consistency. For FID calculations, while SEVA uses its default $299 \times 299$ resolution (aligned with the standard Inception-v3 input), all other methods, including our own, are evaluated at a resolution of $256 \times 448$ to match the native aspect ratio of the datasets. For the FVD metric, all compared methods are evaluated at a unified resolution of $224 \times 224$ using the standard I3D backbone to ensure a consistent and unbiased comparison of temporal dynamics.

\section{Additional Results}
\label{sec:supp:results}
This section provides supplementary qualitative evidence to further substantiate our experimental findings and theoretical analysis.

\subsection{Extended Comparisons in Long-Trajectory Generation}
We present additional visual comparisons to further demonstrate the superiority of the DCARL framework in long-term video synthesis. As illustrated in~\cref{fig:sup_main_qual_res} and~\cref{fig:nusc_sup_res_qual}, our method maintains significantly higher temporal consistency and trajectory fidelity over extended horizons compared to existing baselines. While competing methods often suffer from visual drifting or structural collapse as the sequence progresses, DCARL leverages sparse keyframe anchors to ensure stable and realistic scene evolution throughout the entire trajectory.

\subsection{Comparisons in Small-Scene Generation}
To further evaluate the generalization capability of DCARL, we conduct experiments on the DL3DV-10K dataset~\cite{dl3dv}, which focuses on detailed small-scene reconstruction and generation. The quantitative results are summarized in~\cref{tab:dl3dv_res}.

As illustrated in~\cref{fig:dl3dv_qual}, our method significantly outperforms autoregressive baselines in maintaining long-term visual stability and trajectory adherence. We observe that while existing AR methods often suffer from cumulative exposure bias—leading to gradual semantic drift or visual degradation over extended sequences—DCARL effectively eliminates such artifacts, preserving consistent scene appearance throughout the rollout. Furthermore, our model exhibits superior alignment with the input camera trajectories. While baseline methods frequently deviate from the intended path or exhibit trajectory drift during complex maneuvers, DCARL's keyframe-anchored strategy ensures that the generated views remain strictly faithful to the provided motion constraints, ensuring high-fidelity controllable generation.

\begin{table*}[t]
\centering
\caption{Comparison of our method with state-of-the-art approaches for 32-second video generation on the DL3DV-10K dataset.}
\label{tab:dl3dv_res}
\small
\resizebox{\linewidth}{!}{
\begin{tabular}{lccccccc}
\toprule
Method & Overall & 0-8 s & 8-16 s & 16-24 s & 24-32 s & ATE $\downarrow$ & ARE $\downarrow$ \\ 
\midrule
DiffF & 36.5 / 450.1 & 32.3 / 397.0 & 45.9 / 576.4 & 49.8 / 630.1 & 54.6 / 687.9 & 0.517 & 41.467 \\ 
SelfF & 50.4 / 1062.2 & 46.3 / 571.5 & 56.1 / 877.7 & 78.2 / 1569.3 & 105.2 / 2374.8 & 0.699 & 75.676 \\ 
DeepF & 58.0 / 1279.5 & 37.4 / 479.7 & 73.0 / 1454.6 & 83.4 / 1905.0 & 85.5 / 1988.0 & 0.664 & 66.689 \\ 
SEVA & 24.8 / 275.0 & 35.9 / 290.6 & 40.8 / 390.0 & 39.0 / 371.8 & 38.0 / 397.0 & 0.670 & 68.129 \\ 
\textbf{Ours} & \textbf{22.8 / 220.1} & \textbf{29.1 / 294.6} & \textbf{33.3 / 344.0} & \textbf{30.1 / 324.1} & \textbf{33.3 / 320.1} & \textbf{0.340} & \textbf{19.646} \\ 
\bottomrule
\end{tabular}
}
\end{table*}

\subsection{Failure Cases}
Despite the robust performance of DCARL, we identify certain scenarios where the model still faces challenges, as illustrated in Fig.~\ref{fig:failure_cases}. These failure cases provide insights into current limitations and future research directions.

\noindent\textbf{Long-distance Perception Limits.} The first and second rows of Fig.~\ref{fig:failure_cases} demonstrate that there is still room for improvement in long-range spatial perception. In the first row, the model fails to correctly anchor the vegetation, which should ideally remain within the center of the roundabout. In the second row, the generated trajectory incorrectly passes through a pedestrian island, indicating a lack of fine-grained understanding of complex road boundaries over long distances.

\noindent\textbf{Corner Cases and Data Distribution.} The third row depicts a challenging corner case where the sequence originates from beneath an overpass. Upon exiting the overpass, the model produces unrealistic and unnatural visual artifacts. Such failures suggest that the model remains sensitive to rare structural transitions and lighting changes. 

These observations highlight the necessity of designing more sophisticated long-distance perception mechanisms and further enriching the training set with a more balanced and diverse distribution of corner cases.

\section{Discussion}
\label{sec:supp:limit}
While DCARL effectively addresses the cumulative divergence in long-term video synthesis, we acknowledge several limitations that provide avenues for future research.

\noindent\textbf{Real-time Generation.} Although DCARL does not introduce significantly more inference latency compared to standard bidirectional attention-based models, it is not yet optimized for real-time interactive applications. However, our framework offers a promising blueprint for real-time streaming synthesis. Specifically, future iterations could employ a \textit{dual-stream} architecture: a high-speed causal model could handle low-latency frame generation, while a secondary sparse model generates and maintains a global scene representation—either as a collection of memory frames or 3D primitives—to provide structural anchors. Our work demonstrates that this divide-and-conquer approach is instrumental in stabilizing such real-time autoregressive processes.

\noindent\textbf{Keyframe Generation Quality.} Currently, the generation quality of sparse keyframes is slightly constrained by the inherent difficulty of learning long-range data distributions compared to short-window transitions. This challenge manifests in maintaining high-frequency visual details and precise camera trajectory adherence over ultra-long horizons. Under limited computational and data resources, we believe the dual-generator design remains optimal, and the keyframe generator could be further enhanced by distilling knowledge from 3D foundation models. While a unified generator representation might potentially yield higher quality, it would require significantly larger compute budgets and more diverse data representations. At present, DCARL ensures robust 32-second video generation with high-fidelity camera control and visual stability.

\begin{figure}[t] 
    \centering 
    \includegraphics[width=0.95\linewidth]{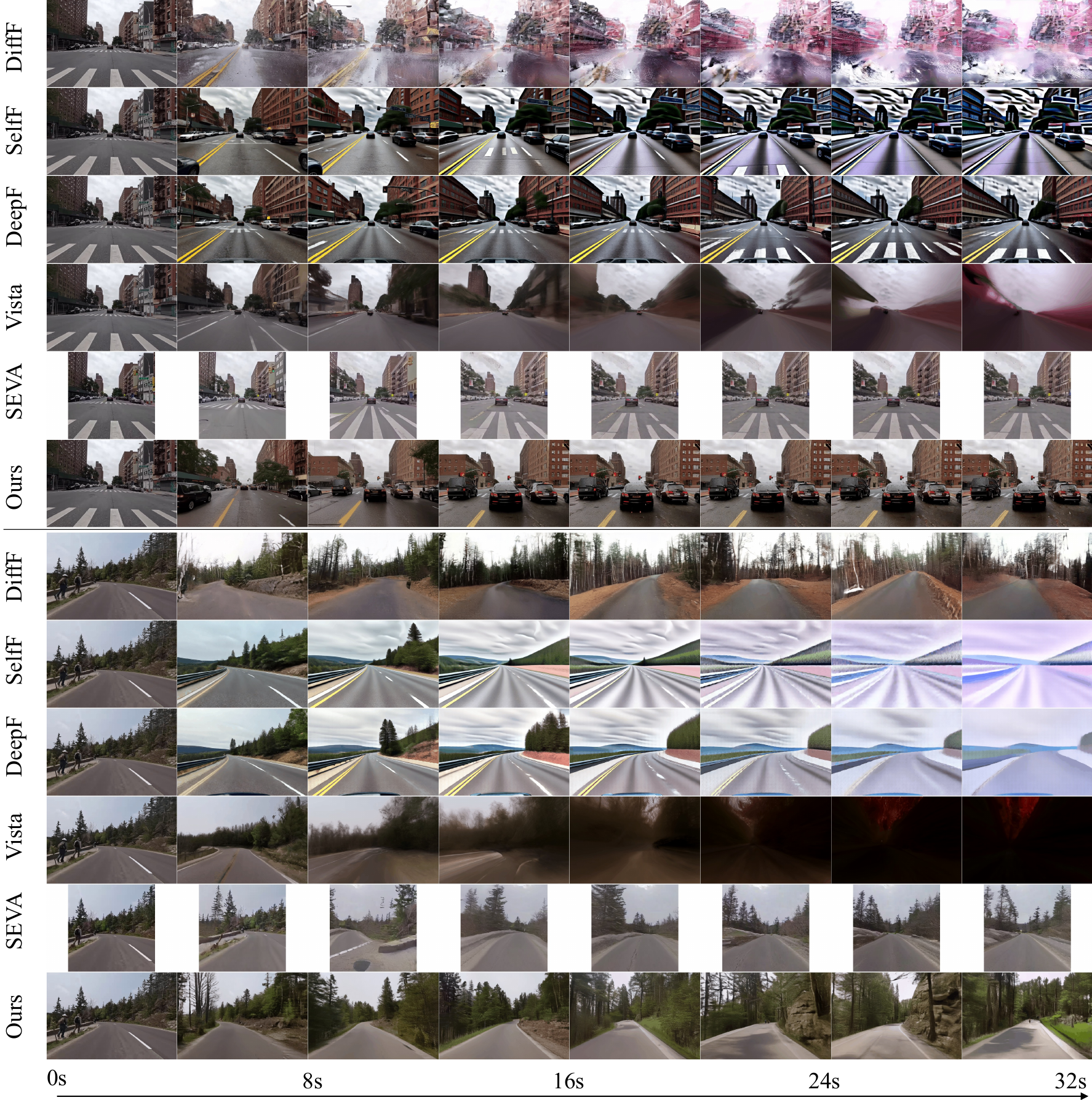} 
    \caption{Visual comparison of 32-second video generation on the OpenDV-YouTube~\cite{opendv} dataset.} 
    \label{fig:sup_main_qual_res} 
\end{figure}

\begin{figure}[t] 
    \centering 
    \includegraphics[width=0.95\linewidth]{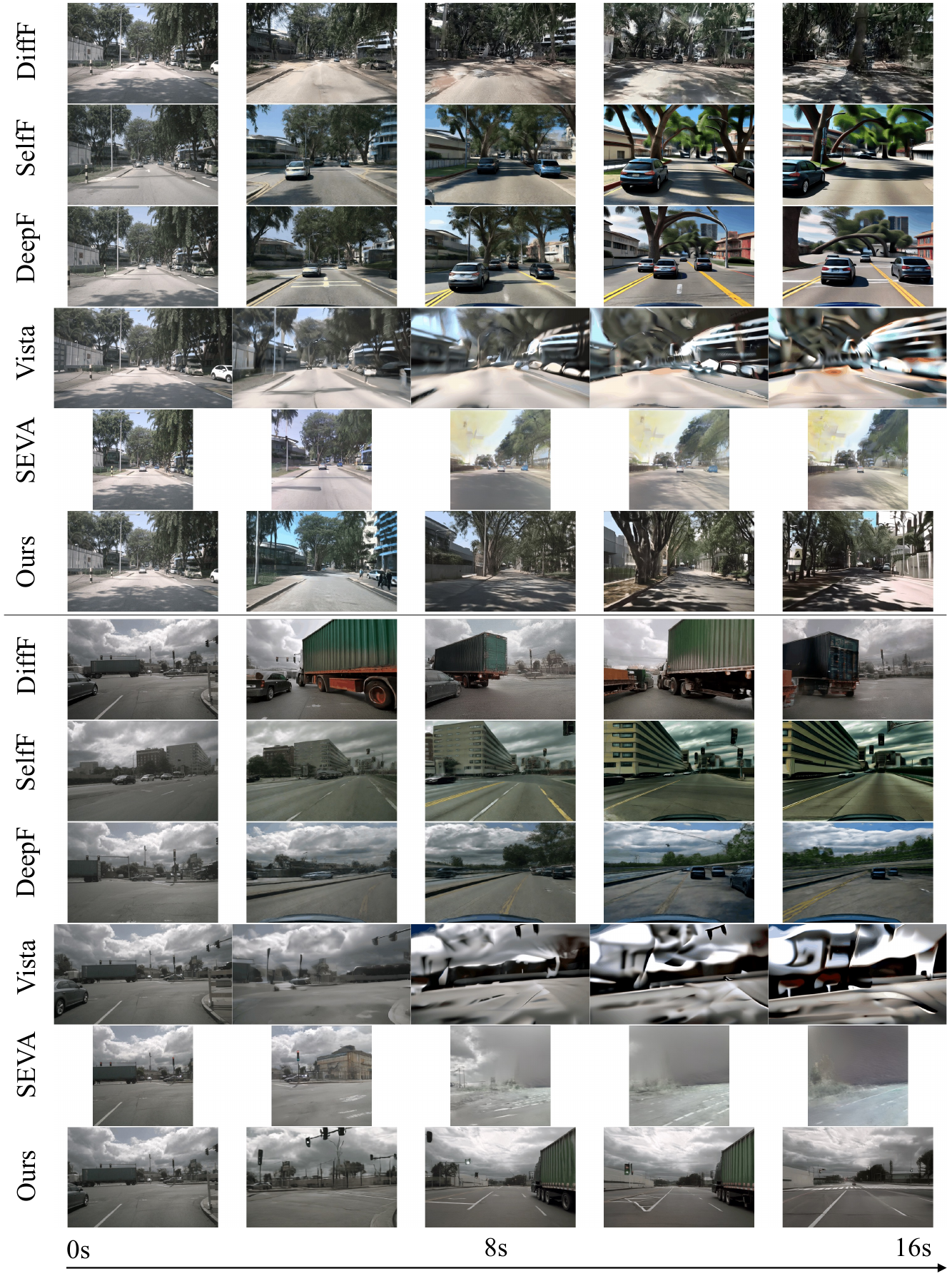} 
    \caption{Visual comparison of 16-second video generation on the NuScenes~\cite{nuscenes2019} dataset.} 
    \label{fig:nusc_sup_res_qual} 
\end{figure}

\begin{figure}[t]
    \centering
    \includegraphics[width=0.98\linewidth]{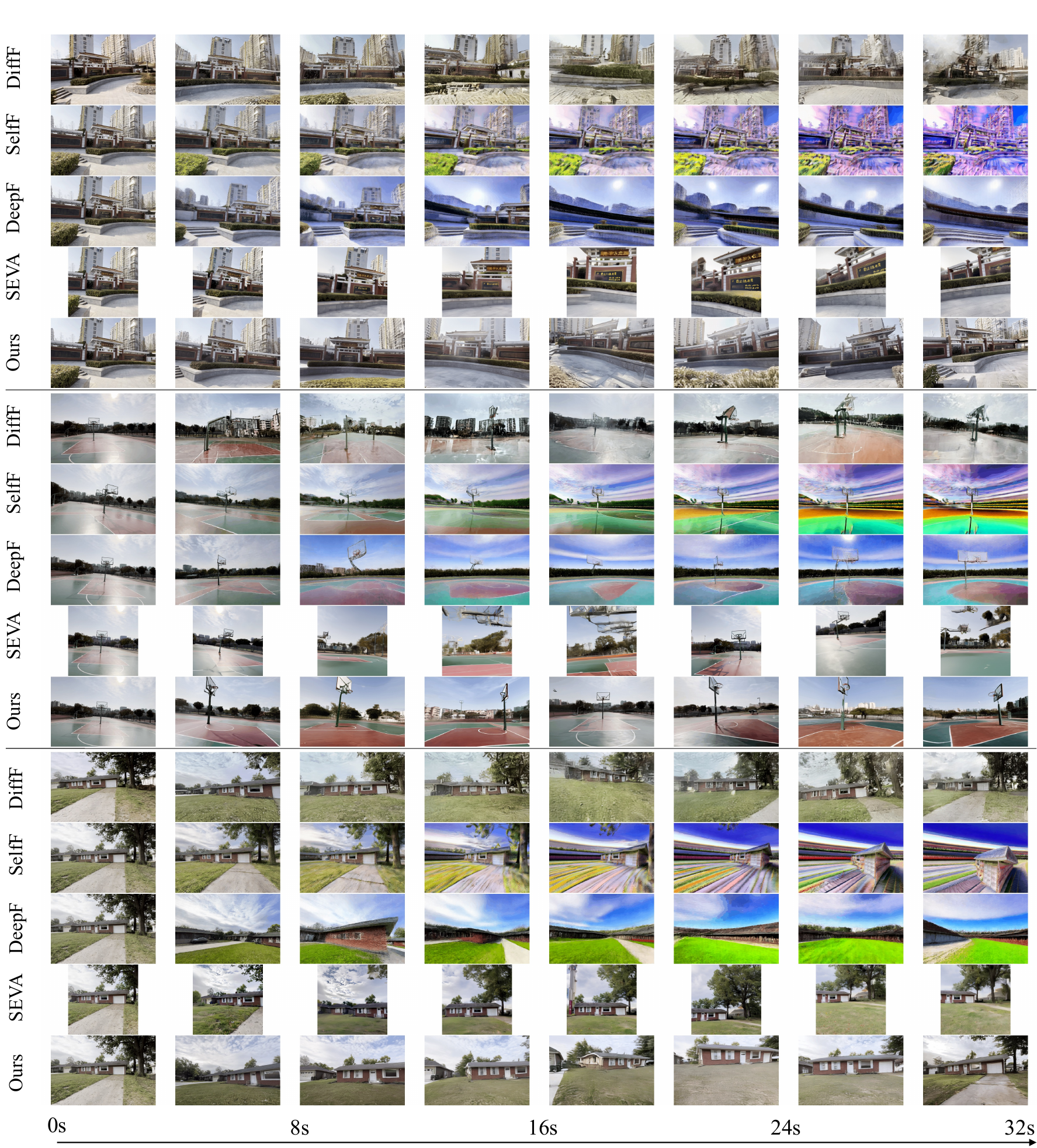}
    \caption{Qualitative comparison of 32-second video generation on the DL3DV dataset. Compared to other methods, DCARL preserves better stability and high-frequency details in small-scene environments.}
    \label{fig:dl3dv_qual}
\end{figure}

\begin{figure}[t]
    \centering
    \includegraphics[width=0.95\linewidth]{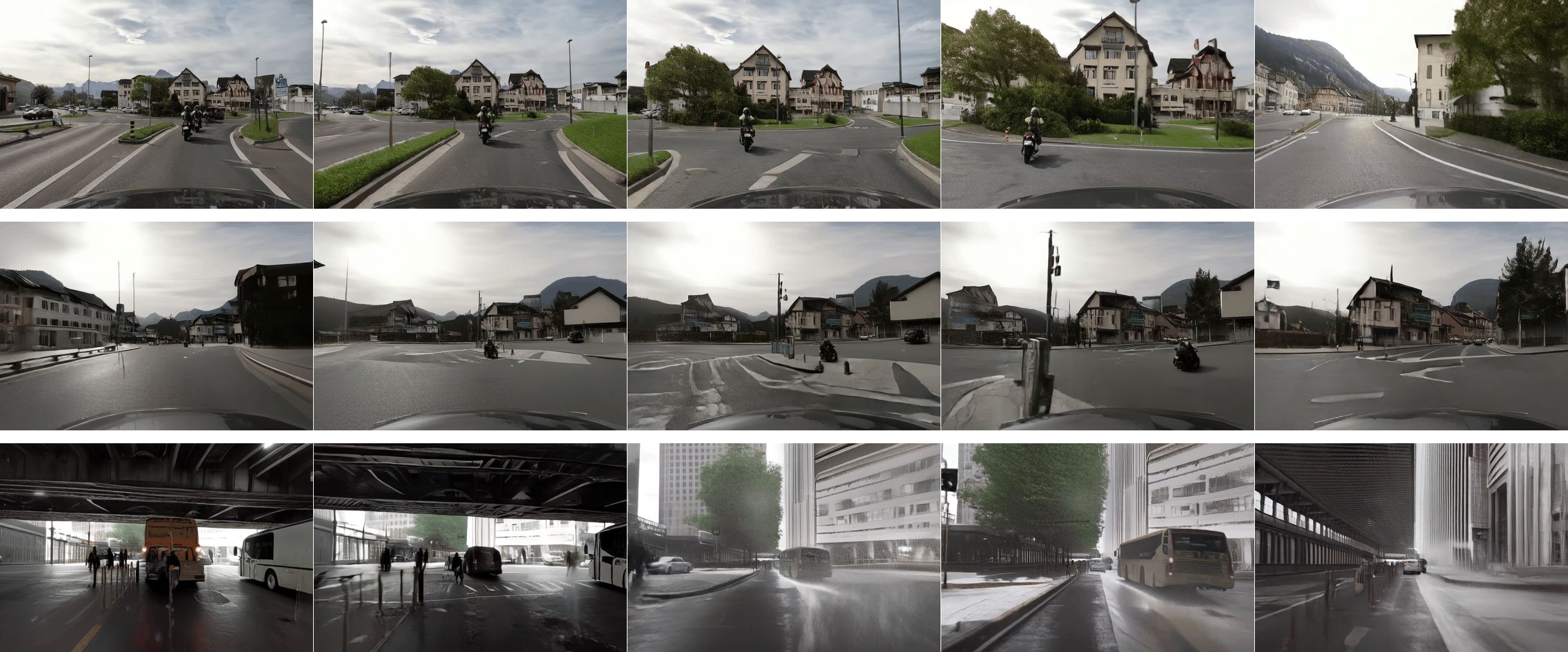}
    \caption{Visual analysis of failure cases. Rows 1-2: Challenges in long-distance perception, such as maintaining roundabout structures and respecting pedestrian islands. Row 3: A corner case involving complex structural transitions (exiting an overpass) that leads to unnatural artifacts.}
    \label{fig:failure_cases}
\end{figure}

\clearpage

\end{document}